\documentclass{article}
\pdfoutput=1
\usepackage{amsmath, amsthm, amssymb}

\usepackage{xspace}
\usepackage{amsmath, amssymb, dsfont}
\usepackage{balance}

\usepackage{graphicx, color}
\usepackage{dsfont}
\usepackage[algo2e,ruled,vlined,linesnumbered]{algorithm2e}
\SetKwFor{Function}{function}{is}{end function}

\SetKw{Input}{input}
\SetKw{Output}{output}

\newcommand{\OneOne}{(1+1)~GP}
\SetKw{Initialization}{initialization}
\SetAlgoCaptionLayout{small}

\newtheorem{theorem}{Theorem}
\newtheorem{lemma}[theorem]{Lemma}

\newtheorem{algorithm}[theorem]{Algorithm}
\newcommand{\oneonegp}{(1+1)~GP\xspace}

\newcommand{\hvlMutateTwoPointOh}{HVL-Prime\xspace}
\newcommand{\oneonegpmulti}{(1+1)~GP-multi\xspace}

\newcommand{\oneonegpsingle}{(1+1)~GP-single\xspace}
\newcommand{\smogp}{SMO-GP\xspace}

\newcommand{\smogpsingle}{SMO-GP-single\xspace}
\newcommand{\smogpmulti}{SMO-GP-multi\xspace}

\newcommand{\MOORDER}{MO-ORDER\xspace}
\newcommand{\MOMAJ}{MO-MAJ\-ORITY\xspace}
\newcommand{\WMOORDER}{MO-WORDER\xspace}
\newcommand{\WMOMAJ}{MO-WMA\-JORITY\xspace}

\newcommand{\ORDER}{ORDER\xspace}
\newcommand{\MAJ}{MA\-JORITY\xspace}
\newcommand{\WORDER}{W\-ORDER\xspace}
\newcommand{\WMAJ}{WMA\-JORITY\xspace}

\newcommand{\Tinit}{T_{init}}

\newcommand{\ignore}[1]{}

\begin{document}


\title{Computational Complexity Analysis of  Multi-Objective Genetic Programming}

\author{
Frank Neumann\\
       School of Computer Science\\
       University of Adelaide\\
       Adelaide, SA 5005, Australia
}



\maketitle

\begin{abstract}
The computational complexity analysis of genetic programming (GP) has been started recently in \cite{GPFOGA11} by analyzing simple \oneonegp algorithms for the problems \ORDER and \MAJ.
In this paper, we study how taking the complexity as an additional criteria influences the runtime behavior.
We consider generalizations of \ORDER and \MAJ and present a computational complexity analysis of  \oneonegp using multi-criteria fitness functions that take into account the original objective and the complexity of a syntax tree as a secondary measure. Furthermore, we study the expected time until population-based multi-objective genetic programming algorithms
have computed the Pareto front when taking the complexity of a syntax tree as an equally important objective.
\end{abstract}


	\section{Introduction}
\label{sec:intro}

Genetic programming (GP)~\cite{koza:1992:book} is an evolutionary computation approach that evolves computer programs for a given task. This type of algorithm has been shown to be very successful in various fields such as symbolic regression, financial trading, and bioinformatics. We refer the interested reader to Poli et al.~\cite{poli08:fieldguide} for a detailed presentation of GP.
Various approaches such as schema theory, markov chain analysis, and approaches to measure problem difficulty have been used to tackle GP from a theoretical point of view~\cite{PoliVLM10}. Poli et al. \cite{PoliVLM10} state explicitly that they expect to see computational complexity results of genetic programming in the near future.

With this paper we start the computational complexity analysis of multi-objective genetic programming. 
This type of analysis has significantly increased the theoretical understanding of other types of evolutionary algorithms (see the books \cite{BookAugDoe,BookNeuWit} for a comprehensive presentation). For various combinatorial optimization problems such as minimum spanning trees~\cite{NeumannWegenerGECCO2005}, minimum multi-cuts~\cite{DBLP:conf/gecco/NeumannRS08,NeumannRPPSN08}, and covering problems~\cite{DBLP:conf/gecco/FriedrichHNHW07}, it has been shown that multi-objective models provably lead to more efficient evolutionary algorithms. 

Initial steps in the computational complexity  analysis of genetic programming have been made by Durrett et al.~\cite{GPFOGA11}. They have studied simple mutation-based genetic programming algorithms on the problems ORDER and MAJORITY introduced in \cite{GoldbergO98}.
Furthermore, the computational complexity of GP has been studied in the  PAC learning framework~\cite{DBLP:conf/gecco/KotzingNS11} defined by Vailiant~\cite{Valiant84} and for the Max Problem~\cite{MaxGECCO12} introduced by Gathercole and Ross~\cite{Gat-Ros:c:96}.

Classical genetic programming often suffers from the occurrence of bloat~\cite{DBLP:conf/eurogp/LangdonP98,DBLP:conf/eurogp/DignumP08}, i.e.\ the growth of parts in the syntax tree that does not have any contribution to the functionality of the program. Due to this, different mechanisms for handling bloat are often incorporated in GP algorithms (see e.g.~\cite{DBLP:journals/ec/LukeP06,DBLP:conf/gecco/PoliM08,DBLP:journals/ec/Alfaro-CidMVES10}).
A simple approach to deal with bloat in GP is to favor solutions of lower complexity if two solutions have the same function value with respect to the given objective function. This leads to a multi-criteria fitness function which is composed of the original function to be optimized and a function assigning a complexity value to a given solution. Still there is a total ordering on the set of possible solutions and a solution would be considered as optimal if it has the smallest complexity among all solutions that achieve the highest function value with respect to the original goal function.

 Another way of dealing with the bloat problem is to use a multi-objective approach where the original function and the complexity are equally important. This induces a partial order on the set of possible solutions as usually the original function and the complexity trade-off against each other.
An advantage of this approach is that solutions of different complexity are generated which gives practitioners insights on how quality trades off against complexity. 
Such an approach is taken in one of the most popular genetic programming tools called DataModeler~\cite{dataModeler:2010:manual} which allows the user to compute the trade-offs with respect to the quality of the model and its complexity.
In this case, often not the best solution with respect to the original function is used but a solution that is still of good quality while having a lower complexity. 

We introduce a population-based genetic programming algorithm for multi-objective optimization called \smogp that is motivated by the computational complexity analysis of an evolutionary multi-objective algorithm called SEMO. This algorithm has been considered in several computational complexity studies for binary search spaces~\cite{DBLP:journals/tec/LaumannsTZ04,Giel2003,NeumannWegenerGECCO2005,
DBLP:conf/gecco/NeumannRS08,NeumannRPPSN08,DBLP:conf/gecco/FriedrichHNHW07,GieLehECJ}. \smogp starts with a single solution, produces in each iteration one offspring, and stores the set of different trade-offs with respect to the given objective functions in the population. 
We study the effect of using the mentioned multi-objective approach in genetic programming in a rigorous way. To do this, we study the computational complexity of \smogp with respect to the runtime that it requires to achieve the so-called Pareto front which is the set of all possible trade-offs of the original given function and the complexity measure.

Throughout this paper, we consider the problems Weighted \ORDER (\WORDER) and Weighted \MAJ (\WMAJ). 
These are generalizations of \ORDER and \MAJ  which have been analyzed in \cite{GPFOGA11}.
This generalization is similar as the generalization of OneMax to the class of linear pseudo-Boolean functions in the investigations of evolutionary algorithms working on binary strings~\cite{DJWoneone}. The analysis of linear pseudo-Boolean has played a key role in the analysis of evolutionary algorithms working on binary string~\cite{WittSTACS12,DBLP:conf/gecco/DoerrJW10,DBLP:journals/algorithmica/Jagerskupper11}. This class of functions has also been examined in the context of ant colony optimization, but determining the exact optimization time of simple ACO algorithms for this class of functions is still a challenging open problem~\cite{DBLP:conf/foga/KotzingNSW11}.

We think that understanding the behavior of simple GP algorithms on \WORDER and \WMAJ will play a similar role in the computational complexity analysis of GP. In this paper, we present first steps in understanding the behavior of simple GP algorithms for these problems. In many cases, we consider GP algorithms carrying out one single mutation operation in each mutation step. This is comparable to randomized local search for binary strings. Our analyses provide important insights for the combination of the original function value and the complexity of the tree. We explicitly state that it is very interesting and challenging to analyze GP algorithms where a larger number of operations is possible in the mutation steps and list such topics for future work in the conclusions.

The outline of the paper is as follows. In Section~\ref{sec2}, we introduce the problems that we consider in this paper. Section~\ref{sec3}, examines the impact of the complexity as a secondary measure and presents runtime analyses for \oneonegp on \WORDER and \WMAJ. In Section~\ref{sec4}, we turn to multi-objective optimization and analyze the time until \smogp has computed the whole Pareto front. We finish with some conclusions and topics for future work.

\section{Preliminaries}
\label{sec2}
We consider tree-based genetic programming, where a possible solution to a given problem is given by a syntax tree. The inner nodes of such a tree are labelled by function symbols from a set $F$ and the leaves of the tree are labelled by terminals from a set $T$.

\begin{figure}[t]
\begin{center}
\includegraphics[scale=.3]{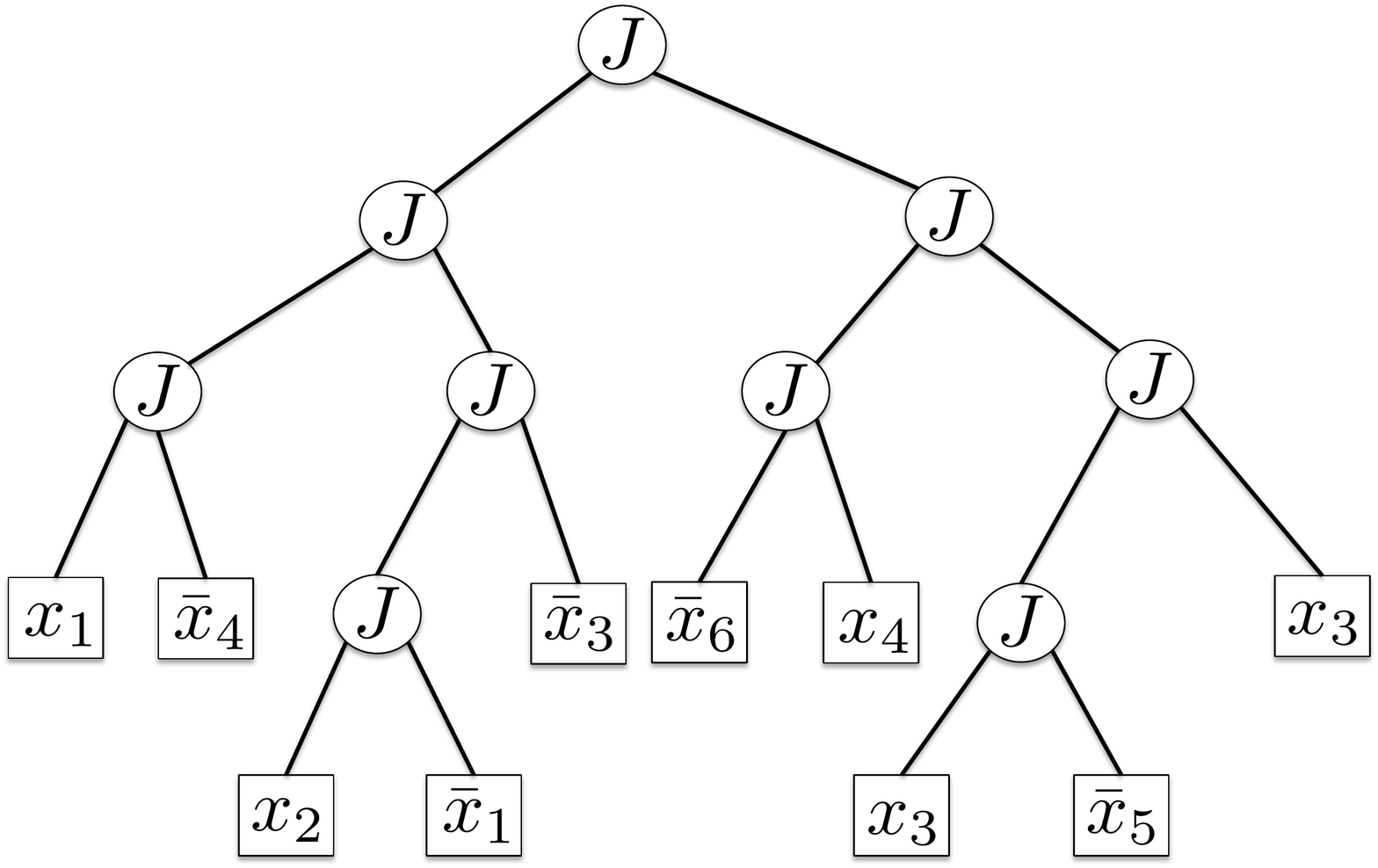}
\end{center}
\vspace{-1.3cm}
\caption{Example tree $X$ with $C(X)=19$}
\label{fig:tree}
\end{figure}

We examine the problems Weighted \ORDER (\WORDER) and Weighted \MAJ (\WMAJ)
which are generalizations of \ORDER and \MAJ analyzed in \cite{GPFOGA11}. 
For both, the only function is the join operation (denoted by $J$).  The terminal set $T$ is a set of $2n$  variables, where $\bar{x}_i$ is the complement of $x_i$:

\begin{itemize}
\item $F:=\{J\}$,  $J$ has arity 2.
\item $T:= \{x_1, \bar{x}_1, \ldots, x_n, \bar{x}_n\}$
\end{itemize}

A valid tree for $n=6$ is shown in Figure~\ref{fig:tree}.
We attach to each variable $x_i$ a weight $w_i \in \mathds{R}$, $1 \leq i \leq n$, such that the variables can differ in their contribution to the overall fitness of a tree. Without loss of generality, we assume that $w_1 \geq w_2 \geq \ldots \geq w_n >0$ holds throughout this paper. This assumption allows for an easier presentation, but is no restriction to the general case as our algorithms treat positive and negative variables in the same way, and do not give preference to any specific variable.

\begin{figure}[t]
 \fbox{
   \begin{minipage}{\columnwidth}
     \vspace{5pt}
Init: $l$ an empty leaf list, $S$ is an empty statement list.
\begin{enumerate}
\item Parse the tree X inorder and insert each leaf at the rear of $l$ as it is visited.
\item Generate $S$ by parsing $l$ front to rear and adding (``expressing'') a leaf to $S$ only if it or its complement are not yet in $S$ (i.e. have not yet been expressed).  
\item \WORDER(X)= $\sum_{x_i \in S} w_i$.
\end{enumerate}
 \vspace{5pt}
   \end{minipage}
 }

\caption{Computation of \WORDER(X)}
\label{fig:ord}

\end{figure}

\begin{figure}[t]
\fbox{
   \begin{minipage}{\columnwidth}
     \vspace{5pt}
Init: $l$ an empty leaf list, $S$ is an empty statement list.
\begin{enumerate}
\item Parse the tree X inorder and insert each leaf at the rear of $l$ as it is visited.
\item For $i \leq n$:  if  $\text{count}(x_{i} \in l) \geq \text{count}(\bar{x}_i \in l)$ and $\text{count}(x_{i} \in l) \geq 1$, add $x_{i}$ to $S$
\item Return \WMAJ(X)= $\sum_{x_i \in S} w_i$.
\end{enumerate}
 \vspace{5pt}
   \end{minipage}
 }
\caption{Computation of \WMAJ(X)}
\label{fig:maj}
\end{figure}

For a given syntax tree $X$, the value of the tree is computed by parsing the tree inorder. The weight $w_i$ of a variable $x_i$ contributes to the fitness iff $x_i$ is positive and contained in the set $S$ of the evaluation function.
For \WORDER $x_i$ is contained in $S$ iff it is present in the tree and there is no $\bar{x}_i$ that is visited in the inorder parse before $x_i$.
For \WMAJ, $x_i$ is contained in $S$ iff $x_i$ is present in the tree and the number of $x_i$ variables in $X$ is at least as high as the number of $\bar{x}_i$ variables in $X$. 
For a given tree $X$ their evaluation is shown in Figures~\ref{fig:ord} and \ref{fig:maj}. 
\ORDER and \MAJ as special cases where $w_i=1$, $1\leq i \leq n$, holds.

We illustrate both problems by an example.
Let $n=6$ and $w_1=13$, $w_2 =11$, $w_3=8$, $w_4=7$, $w_5 = 5$, $w_6=3$.
 For the tree  X show in Figure~\ref{fig:tree}, we get (after the inorder parse)  

\[
l = (x_{1}, \bar{x}_4, x_{2}, \bar{x}_1, \bar{x}_{3}, \bar{x}_6, x_4, x_3, \bar{x}_{5}, x_3)
\]

For \WORDER, we get
  $S=(x_{1}, \bar{x}_4, x_{2}, \bar{x}_{3}, \bar{x}_6, \bar{x}_5)$ and 
\[
\text{\WORDER}(X)= w_1 + w_2 = 13+11=24.
\]

For \WMAJ, we get 
 $S=(x_{1}, x_{2}, x_{3}, x_4)$ and 
\[
\text{\WMAJ}(X)= w_1 + w_2 + w_3 +w_4 = 13+11+8+7=39.
\]
The complexity $C$ of a given tree $X$ is the number of nodes it contains. For the tree $X$ given in Figure~\ref{fig:tree}, $C(X) =19$ holds.

There are two problems we will consider. The first one is the single-objective problem of one computing a solution $X$ which maximizes $F$. During the optimization run, our algorithms are allowed to use the function $C$ as an additional criteria if two solutions have the same function value with respect to $F$. The second problem is the computation of the Pareto front for the multi-objective problem given by $F$ and $C$.

We study genetic programming algorithms which take into account the originally given problem as well as the complexity of a given solution. We can formulate this as a multi-objective problem which assigns different objective values to a given solution.
Throughout this paper, we assume that we have one objective function $F$ that should be maximized and have the complexity $C$ of a GP-syntax tree as the second objective which should be minimized. 
$F$ can be considered as the original problem at hand, and the minimization of $C$ allows to cope with the bloat problem.
Our algorithms work with the multi-criteria fitness function MO-F(X)= (F(X), C(X)).

Consequently, we obtain the following problems when ad\-ding the complexity of a solution $X$ as the second criteria.
\begin{itemize}
\item \WMOORDER(X) = (\WORDER(X), C(X))
\item \WMOMAJ(X) = (\WMAJ(X), C(X))
\end{itemize}

For the special case where $w_i=1$, $1 \leq i \leq n$, holds, we obtain the problems

\begin{itemize}
\item \MOORDER(X) = (\ORDER(X), C(X))
\item \MOMAJ(X) = (\MAJ(X), C(X))
\end{itemize}

which add the complexity $C$ as an additional objective to the problems \ORDER and \MAJ. We will pay special attention to these problems and examine how the use of the additional complexity objective influences the runtime behavior as it allows a direct comparison to the results obtained in \cite{GPFOGA11}. Note, that an alternative way of modelling our problems is to work directly with the weights $w_i$ and $\bar{w}_i$, $1 \leq i \leq n$, as variables in the tree. Such a presentation is equivalent to the one we have chosen and would lead to the same results as presented in this paper.

\begin{figure}[t]
\fbox{
   \begin{minipage}{\columnwidth}
     \vspace{5pt}
Mutate $Y$ by applying \hvlMutateTwoPointOh $k$ times. For each application, randomly choose to either substitute, insert, or delete.

\begin{itemize}
\item If substitute, replace a randomly chosen leaf of $Y$ with a new leaf $u \in T$ selected uniformly at random.
\item If insert, choose a node $v$ in $Y$ uniformly at random and select $u \in T$ uniformly at random. Replace $v$ with a join node whose children are $u$ and $v$, with the order of the children chosen randomly.
\item If delete, randomly choose a leaf node $v$ of $Y$, with parent $p$ and sibling $u$. Replace $p$ with $u$ and delete $p$ and $v$.
\end{itemize}
 \vspace{5pt}
   \end{minipage}
 }
\caption{Mutation operator}
\label{fig:mut}
\end{figure}

We consider simple mutation-based genetic programming algorithms. They use the operator \hvlMutateTwoPointOh which has been part of the (1+1)~GP variants analyzed in  \cite{GPFOGA11}. \hvlMutateTwoPointOh allows to produce trees of variable length and is based on three different operations, namely insert, substitute and delete. Each application of \hvlMutateTwoPointOh chooses one of these operations randomly. Throughout this paper, randomly chosen always means randomly chosen with respect to the uniform distribution. The complete description of the mutation operator is given in Figure~\ref{fig:mut}. For its application a parameter $k$ determining the number of \hvlMutateTwoPointOh operations has to be chosen. As in \cite{GPFOGA11}, we consider two possibilities. In the case of single-operations $k=1$ holds. For multi-operations $k$ is chosen according to $1+ Pois(1)$ where $Pois(1)$ denotes the Poisson distribution with parameter $\lambda=1$.

\section{(1+1)~GP}
\label{sec3}

In this section, we consider \oneonegp algorithms working with the multi-criteria fitness functions introduced in the previous section. 
The algorithms are simple hill-climbers that explore their neighbourhood in dependence of the mutation operator.
They differ from the ones analyzed in \cite{GPFOGA11} only in the selection step. 
 The outline of \oneonegp is shown in Algorithm~\ref{alg:rls}. It starts with an initial solution $X$ and produces in each iteration one single offspring $Y$ by mutation. $Y$ replaces $X$ if it is favored according the selection mechanism. 
\begin{algorithm}[\oneonegp]
\label{alg:rls}
\begin{enumerate}
\item[]
\item Choose an initial solution $X$.
\item Repeat
\begin{itemize}
\item Set $Y:= X$.
\item Apply mutation to $Y$.
\item If selection favors $Y$ over $X$
then $X := Y$.
\end{itemize}
\end{enumerate}
\end{algorithm}

We will consider the algorithm \oneonegpsingle which applies the mutation operator \hvlMutateTwoPointOh once in each mutation step, i.e. the mutation operator given in Figure~\ref{fig:mut} is used for $k=1$.
Analyzing the computational complexity of this algorithm, we are interested in the expected number of fitness evaluations until the algorithm has found an optimal solution for the given problem $F$ for the first time. This is called the \emph{expected optimization time} of the analyzed algorithm. 

The worst case results for \oneonegpsingle obtained in \cite{GPFOGA11} depend on the maximum size of the tree (denoted by $T_{\max}$) that is encountered during the optimization process. To be more precise, the upper bound for \oneonegpsingle is $O(nT_{\max})$ for \ORDER and $O(n^2 T_{\max} \log \log n)$ for \MAJ.
As $T_{\max}$ is not known in advance, it is more desirable to have runtime bounds that only depend on the input and the size of the initial tree. In such a case, the user has complete knowledge on how much worse such a bound can get. Especially, in the light of the bloat problem, $T_{\max}$ can be assumed to be quite large for various types of problems. We will analyze our algorithms in dependence of the tree size of the initial solution (denoted by $T_{init}$). 

The key point of our study is to examine how the complexity of a solution as the secondary measure influences the runtime. The selection mechanism for the \oneonegpsingle variant studied in \cite{GPFOGA11} (\oneonegpsingle on F) and the selection in our algorithm  (\oneonegpsingle on MO-F) are shown in Figure~\ref{fig:ea}. Note, that using \oneonegpsingle on MO-F presents a parsimony approach which is quite common in genetic programming to deal with the bloat problem.

\begin{figure}
 \fbox{
   \begin{minipage}{\columnwidth}
     \vspace{5pt}
\label{fig:sel}
\begin{itemize}
\item \oneonegpsingle on $F$: Favor $Y$ over $X$ iff $$F(Y) \geq F(X).$$
\item \oneonegpsingle on MO-F: Favor $Y$ over $X$ iff $$(F(Y) > F(X)) \vee ((F(Y)=F(X)) \wedge (C(Y) \leq C(X))).$$
\end{itemize}
\vspace{5pt}
   \end{minipage}
 }
\caption{Selection for \OneOne}
\label{fig:ea}
\end{figure}

\subsection{Analysis}

We start our analysis of \oneonegpsingle by presenting a general lower bound on the expected optimization time.  This bound holds independently of the chosen fitness function and is a direct consequence of the coupon collector's theorem~\cite{MotwaniRaghavan}. 
\begin{theorem}
\label{thm:lboneone}
Let $X$ be the empty tree, then the expected time until \oneonegpsingle  has produced an optimal solution for \MOORDER and \MOMAJ is $\Omega(n \log n)$. 
\end{theorem}

\begin{proof}
In order to produce an optimal solution for the given problems, each positive variable has to be introduced at least once into the tree. The probability to introduce one specific variable $x_i$ in the next step is at most $\frac{1}{3} \cdot \frac{1}{n}$. Using the coupon collector's theorem, the result follows immediately.
\end{proof}

Theorem~\ref{thm:lboneone} shows that we can not expect a better upper bound then $O(n \log n)$. This is a typical bound for many simple evolutionary algorithms as they usually encounter the coupon collector effect. 
In the following, we present upper bounds on the runtime of \oneonegpsingle working with the multi-criteria fitness functions. Theorem~\ref{thm:lboneone} implies that the upper bounds presented in the following are tight.

A variable $x_i$ is called \emph{expressed} if it contributes to the overall fitness of our original problem $F$. This is the case if a variable is positive and contained in the statement list $S$ of our evaluation function.
We call a solution $X$ \emph{non-redundant} if the number of expressed variables is $k$ and its complexity is $2k-1$. Furthermore, the empty tree is called \emph{non-redundant} as well.
 For the problems we consider, any tree that does not fall into the non-redundant category can be improved with respect to complexity without decreasing its fitness. Solutions where such improvements with respect to the complexity are possible are called \emph{redundant}. The key idea of our analysis is to show that the algorithm quickly eliminates redundant variables. After these redundant variables have been removed, the algorithm can introduce missing variables at any position of the tree.

We present upper bounds for  \oneonegpsingle on \WMOORDER and \WMOMAJ which are tight if $T_{init}=O(n \log n)$ holds.
\begin{theorem}
\label{thm:oneoneorder}
The expected optimization of \oneonegpsingle on \WMOORDER is $O(T_{init} + n  \log n)$.
\end{theorem}

\begin{proof}
For our analysis we consider two phases. First we analyze the time until the tree has become \emph{non-redundant}. Afterwards, we bound the time to obtain an optimal solution.

We claim that after an expected number $O(T_{init} + n \log n)$ steps the tree is \emph{non-redundant}.
Let $k$ be the number of expressed variables
and $s$ be the number of leaves in the tree. Then there are $s-k$ variables that can be deleted without changing the \WORDER-value. Such a step reduces the complexity of the tree and is therefore accepted. The probability for such a deletion is at least 
$$
\frac{s-k}{3 \cdot s}.$$

 We show that the value of $s-k$ can not increase during the run of the algorithm. Obviously $k$ can not decrease as selection is primarily based on \WORDER. The number of leaves $s$ can only increase by $1$ if a step is an improvement according to \WORDER. In this case, $s-k$ does not change which shows that $s-k$ does not increase during the run of the algorithm.
Using the method of fitness-based partitions~(see e.g. Chapter $4$ in~\cite{BookNeuWit}) the expected time until  $s=k$ holds is upper bounded by

\begin{eqnarray*}
& & \sum_{j=1}^{T_{init}} \left(\frac{j}{3  \cdot(j+k)}\right)^{-1}\\
 &  = & \sum_{j=1}^{n} \left(\frac{j}{3 \cdot (j+k)} \right)^{-1}
 +  \sum_{j=n+1}^{T_{init}} \left(\frac{j}{3  \cdot (j+k)} \right)^{-1}\\
& \leq & 3 \cdot \sum_{j=1}^{n} \left(\frac{j}{j+n} \right)^{-1}
 +  3 \cdot \sum_{j=n+1}^{T_{init}} \left(\frac{j}{j+n} \right)^{-1}\\
& \leq & 3 \cdot \sum_{j=1}^{n} \left(\frac{j}{j+n} \right)^{-1}
 +3 \cdot \sum_{j=n+1}^{T_{init}} \left(\frac{1}{2} \right)^{-1}\\
& = & 3 \cdot\sum_{j=1}^{n} \frac{j+n}{j}
 + 3 \cdot \sum_{j=n+1}^{T_{init}} 2\\
  & = & O(n \log n) + O(T_{init})
\end{eqnarray*}

Now, we consider the time to reach an optimal solution and work under the assumption that $X$ is non-redundant. Note, that this invariant is maintained as we have shown that the difference $s-k$ can not increase during the run of the algorithm.
Let $n-k$ be the number of unexpressed variables after for the first time a non-redundant tree has been obtained. Any of these $n-k$ variables can be inserted at any position in the tree in order to improve the \WORDER-value. In total, there are $2n$ variables to choose from. Hence, the probability to achieve an improvement is at least 
$$
\frac{1}{3} \cdot \frac{n-k}{2n}.$$

Using again the method of fitness-based partitions, the expected time to achieve an optimal tree which consists of the variables $x_i$, $1 \leq i \leq n$, is upper bounded by
\[\sum_{k=0}^{n-1} \left(\frac{n-k}{6 \cdot n}\right)^{-1} = 6n \cdot\sum_{k=0}^{n-1} \frac{1}{n-k} = O(n \log n).
\]

Summing up the runtimes for the two phases, the expected optimization time of \oneonegpsingle on \WMOORDER  is $O(T_{init} + n  \log n)$.
\end{proof}

We now transfer the previous result to the problem \WMAJ. The analysis carried out in \cite{GPFOGA11} for \oneonegpsingle on \MAJ has to take into account random walk arguments for dealing with plateaus in the search space which leads to a runtime bound of $O(n^2 T_{\max} \log \log n)$ for \MAJ.

Using \WMOMAJ we do not face the difficulty of a plateau during the optimization as the (1+1)~GP variants considered in \cite{GPFOGA11}. The random walk is averted as solutions with the same \WMAJ-value, but a higher complexity are not accepted by the algorithm. In fact, the additional search direction given by the information on the size of the tree leads to a similar fitness landscape as for \WMOORDER. This leads to the following result.

\begin{theorem}
\label{thm:oneonemaj}
The expected optimization time of \oneonegpsingle on \WMOMAJ is $O( T_{init} + n \log n)$.
\end{theorem}
\begin{proof}
The proof of Theorem~\ref{thm:oneoneorder} for \WMOORDER has only used the fact that the difference $s-k$ can not increase during the run of the algorithm and that later on (in the second phase) each non-expressed variable can be inserted at any position in the current tree. Both properties also hold for \WMOMAJ which implies that we get the same upper bound of $O( T_{init} + n \log n)$.
\end{proof}

\section{Multi-Objective Algorithms}
\label{sec4}
The previous section has shown that using the complexity of the syntax tree as a secondary measure can provably lead to better upper bounds on the runtime of simple genetic programming algorithms. 
Depending on the complexity that one allows for a given problem, the value of the best solution for the original problem $F$ may vary. In the case of multi-objective optimization, we are interested in the different trade-offs between the original problem $F$ and the complexity $C$.
In this section, we analyze simple multi-objective genetic programming algorithms until they have computed the whole Pareto front for a given problem MO-F(X) = (F(X), C(X)).

\subsection{Multi-Objective Genetic Programming}
The idea in multi-objective optimization is to treat the given criteria as equally important. 
We consider the following relations on search points which will later on be used in the selection step of our algorithms.
\begin{enumerate}
\item A solution $Y$ \emph{weakly dominates} a solution $X$ (denoted by $Y \succeq X$) iff $(F(Y) \geq F(X) \wedge C(Y) \leq C(X))$.
\item A solution $Y$ \emph{dominates} a solution $X$ (denoted by $Y \succ X$) iff   $(Y \succeq X)  \wedge (F(Y) > F(X) \vee C(Y) < C(X))$.
\item Two solutions $X$ and $Y$ are called \emph{incomparable} iff neither $X \succeq Y$ nor $Y \succeq X$ holds.
\end{enumerate}

A solution is called \emph{Pareto optimal} iff it is not dominated by any other solution in the search space $S$. The set of Pareto optimal solutions is called the \emph{Pareto optimal set} and the set of corresponding objective vectors is called the \emph{Pareto front}. The classical goal in multi-objective optimization is to compute for each objective vector of the Pareto front a corresponding Pareto optimal solution. We introduce and analyze an algorithm called Simple Multi-Objective Genetic Programming (\smogp) which is motivated by the Simple Multi-Objective Optimizer (SEMO) algorithm that has frequently been considered in the computational complexity analysis of evolutionary multi-objective optimization algorithms for binary search spaces ~\cite{DBLP:journals/tec/LaumannsTZ04,Giel2003,NeumannWegenerGECCO2005,
DBLP:conf/gecco/NeumannRS08,NeumannRPPSN08,DBLP:conf/gecco/FriedrichHNHW07,GieLehECJ}.
\smogp starts with a single solution and  produces in each iteration one single offspring $Y$ by mutating an individual of the current population $P$. The population consists in each iteration of a set of solutions that are non-dominated by any other solution seen so far during the run of the algorithm. In the selection step, the offspring $Y$ is added to the population $P$ iff it is not dominated by any other solution in $P$. If $Y$ is added to $P$ all solutions that are weakly dominated by $Y$ are removed from $P$.

\begin{algorithm}{\smogp}
\begin{enumerate}
  \setlength{\itemsep}{2pt}
  \setlength{\parskip}{0pt}
  \item Choose an initial solution $X$.
  \item Set $P := \{X\}$.
  \item Repeat
    \begin{itemize}
      \setlength{\itemsep}{2pt}
      \setlength{\parskip}{0pt}
      \item Choose $X \in P$ uniformly at random.
	\item Set $Y:= X$.
	\item Apply mutation to $Y$.
      \item If $\{Z \in P \mid Z \succ Y\} = \emptyset$,\\ set $P := (P \setminus \{Z \in P \mid Y \succeq Z\}) \cup \{Y\}$.
    \end{itemize}
\end{enumerate}
\end{algorithm}

We consider the algorithms \smogpsingle and \smogpmulti. Both use the mutation operator given in Figure~\ref{fig:mut}. For \smogpsingle k =1 holds, and for \smogpmulti  the parameter $k$ is chosen according to $1+ Pois(1)$.
Our goal is to investigate the expected number of iterations until our algorithms have computed a population which contains for each Pareto optimal objective vector a corresponding solution. We call this the \emph{expected optimization time} of the multi-objective genetic programming algorithms.

Our multi-objective model trades off the function value against the complexity value. A special Pareto optimal solution of the multi-objective model is the empty tree which has the lowest possible complexity value. The following lemma bounds the expected time until the empty tree has been included into the population $P$ when considering an arbitrary problem MO-F.
We denote by $\Tinit$ the size of the tree of the initial solution and analyze the time to include the empty tree in dependence of $\Tinit$ and the number of different fitness values of the problem $F$. 

\begin{lemma}
\label{lem:init}
Let $\Tinit$ be the size of the initial solution and $k$ be the number of different fitness values of a problem F. Then the expected time until the population of \smogpsingle and \smogpmulti applied to MO-F contains the empty tree is $O(k\Tinit)$.
\end{lemma}

\begin{proof}
As the problem F has at most $k$ different fitness values, the population size of the algorithms is bounded by $k$. At each time step  we consider the solution with the lowest complexity in the population. This solution is selected for mutation with probability at least $1/k$. A single deletion operation applied to this individual leads to a new solution of lower complexity. The probability for such a mutation step is at least $1/(3ek)$. Summing up the different values for the minimal tree size in the population, we get
\[
\sum_{i=1}^{\Tinit} 3ek = 3ek\Tinit=O(k\Tinit)
\]
as an upper bound on the expected time until the empty tree is included in the population.
\end{proof}

\subsection{\ORDER and \MAJ}

We now examine how \smogpsingle and \smogpmulti can compute the Pareto front for the multi-objective problems given by \MOORDER and \MOMAJ.
In the following, we show that
both algorithms compute the whole Pareto front for both problems in expected time $O(n \Tinit  + n^2 \log n)$.

We remark that a lower bound of $\Omega(n^2 \log n)$ holds for both algorithms and both problems when starting with the empty tree. This bound can be obtained by using the coupon collector's theorem in a similar way as in Theorem~\ref{thm:lboneone} and taking into account the additional factor of $n$ for the population size.

\begin{theorem}
\label{thm:smoorder}
The expected optimization time of \smogpsingle and \smogpmulti on \MOORDER is $O(n \Tinit  + n^2 \log n)$.
\end{theorem}

\begin{proof}
Due to Lemma~\ref{lem:init}, the empty tree is produced for any MO-F problem having $k$ different fitness values after an expected number of 
$O(k\Tinit)$ steps. The number of different fitness values for \ORDER is $n+1$ which implies that the empty tree is introduced into the population after an expected number of $O(n\Tinit)$ steps. This solution will never be removed from the population as it is the unique solution having complexity $0$.

Assuming that the empty tree has been introduced into the population, we analyze the time until the algorithm has produced solutions that are Pareto optimal and have \ORDER-values $1, 2, \ldots, n$. Each tree having $i$ leaves has exactly $i-1$ inner nodes. Hence, a solution that has \ORDER-value $i$ has complexity at least $2i-1$, $1 \leq i \leq n$. A solution with \ORDER-value $i$ is Pareto optimal iff it has complexity exactly $2i-1$.
We assume that the population contains all Pareto optimal solutions with \ORDER-value $j$, $0 \leq j \leq i$. Then choosing the Pareto optimal solution $X$ with \ORDER(X)$ = i$ for mutation and inserting one of the remaining $n-i$ non-expressed variable, produces a population which includes for each Pareto optimal solutions with ORDER-value $j$, $0 \leq j \leq i+1$, a corresponding solution. Note that this operation produces from a solution of complexity $2i-1$ as solution of complexity $2i-1+2 = 2(i+1)-1$ as an insertion introduces a new leaf and a new joint node.

We have to analyze the probability that such a step happens in the next iteration. Choosing $X$ for mutation has probability at least $1/(n+1)$ as the population size is upper bounded by $n+1$.
A mutation step carrying out just one single operation happens with probability at least $1/e$ and an insertion operation is chosen with probability $1/3$. Finally, $n-i$ variables (among $2n$ variables $T$) can be inserted to produce the Pareto optimal solution of \ORDER-value $i+1$. In total, the probability of producing the Pareto optimal solution of ORDER-value $i+1$ is at least
\[
\frac{1}{n+1} \cdot \frac{1}{3e} \cdot \frac{n-i}{2n}
\]

We use the method of fitness-based partitions according to the different values of $i$. This implies that the expected time until all Pareto optimal solutions have been produced after the empty tree has been included in the population is upper bounded by 

\begin{eqnarray*}
& & \sum_{i=0}^{n-1}  \left(\frac{1}{n+1} \cdot \frac{1}{3e} \cdot \frac{n-i}{2n}\right)^{-1}\\
& = & 6en(n+1) \cdot \sum_{i=0}^{n-1}  \frac{1}{n-i} = O(n^2 \log n)
\end{eqnarray*}

Taking into account the expected time to produce the empty tree, the expected time until the whole Pareto front of \MOORDER has been computed is $O(n \Tinit  + n^2 \log n)$.
\end{proof}

For \MOMAJ we can adapt the ideas of the previous proof. Again the algorithms do not encounter the problem of plateaus in the search space which makes the optimization process much easier than for \MAJ.
\begin{theorem}
The expected optimization time of \smogpsingle and \smogpmulti on \MOMAJ is $O(n \Tinit + n^2 \log n)$.
\end{theorem}

\begin{proof}

For \MOMAJ we can follow the same arguments. The number of different fitness values of \MAJ is upper bounded by $n+1$ and this is an upper bound on the population size. 
The empty tree is produced after an expected number of $O(nT_{init})$ steps according to Lemma~\ref{lem:init}.
Having a population which contains all Pareto optimal solutions with \MAJ-values $j$, $0 \leq j \leq i$, a population which includes for each Pareto optimal solutions with \MAJ-value $j$, $0 \leq j \leq i+1$ is obtained by inserting one of the non-expressed variables into the Pareto optimal individual $X$ with $\MAJ(X)=i$. The probability that such a step happens in the next iteration is at least
\[
\frac{1}{n+1} \cdot \frac{1}{3e} \cdot \frac{n-i}{2n}
\]
and summing up the expected waiting times as done in the proof of Theorem~\ref{thm:smoorder} completes the proof.
\end{proof}

\subsection{Weighted \ORDER and \MAJ}

In our previous investigations of \MOORDER and \MOMAJ each expressed variable contributed an amount of $1$ to the overall fitness of a solution. In this subsection, we extend our investigations to \WMOORDER and \WMOMAJ.

Considering these problems, it is in principle possible to have an exponential number of incomparable solution. Assume for example that $w_i = 2^{n-i}$, $1 \leq i \leq n$ holds. Then there are $2^n$ different fitness values for \WORDER and \WMAJ. Furthermore, one can construct trees for these solutions such that no solution dominates any other solution in this set.

Note, that such a set of solutions does not constitute the Pareto front and that the Pareto front has size $n+1$. As stated in Section~\ref{sec2}, we assume without loss of generality that $w_1 \geq w_2 \geq \ldots \geq w_n>0$ holds in this paper. Then the tree containing exactly the variables $x_1, \ldots, x_i$ is Pareto optimal and has complexity $2i-1$, $1\leq i \leq n$. Furthermore, the empty tree is Pareto optimal which gives us the whole Pareto front of size $n+1$.

We consider the special case, where \smogpsingle starts with a non-redundant solution. We show that in this case \smogpsingle will not accept any redundant solution. This is the key idea for the following theorem.

\begin{theorem}
Starting with a non-redundant initial solution, the expected optimization time of \smogpsingle on \WMOORDER and \WMOMAJ is $O(n^3)$.
\end{theorem}

\begin{proof}
We first study the population size and show that it is in each iteration at most $n+1$.
We claim that the population can only include solutions that have no redundant variables.

The initial solution is non-redundant due to the assumption of the theorem. We prove by induction that this property also holds for all solutions that are later on accepted by the algorithm. Let $X$ be a non-redundant solution of the current population and $Y$ be a redundant offspring created by a single operation. The only operations that can lead to redundant variables in $Y$ are substitute and insert. If substitute introduces a redundant variable, it has to remove a non-redundant variable at the same time. This decreases the fitness while the complexity stays the same. Hence, such a step is not accepted. If an insertion operation introduces a redundant variable then the fitness stays the same and the complexity increases. Such steps are also not accepted. Hence, the algorithm will not accept a redundant solution at any point of the optimization run. There are at most $n$ variables that can be expressed. Hence, the complexity can only take on $n+1$ different values, namely $0, 1, 3, 5, \cdots 2n-1$. This implies that the population size is upper bounded by $n+1$.

We now study how to obtain the different Pareto optimal solutions. We first analyze the expected time until the population contains the empty tree which is the Pareto optimal solution of lowest complexity. 
Let $X$ be the solution in the population that has currently the lowest complexity. If we delete one of the variables we get a new solution $Y$ with $C(Y) < C(X)$. The probability for such a step is at least 
$\frac{1}{3} \cdot \frac{1}{n+1}$
and the expected waiting time to produce such a solution $Y$ is $O(n)$. There are at most $n$ such steps until the empty tree has been reached which implies that the empty tree is included in the population after an expected number of $O(n^2)$ steps.

The empty tree is a Pareto optimal solution as it has complexity $0$. A solution of complexity $2j-1$ is Pareto optimal if the tree contains for the largest $j$ weights exactly one positive variable. 

Let $P$ be a population that contains for all \WORDER-values (same arguments can be used  for \WMAJ-values)
\[
\sum_{k=1}^j w_k,  ~~~~0 \leq j \leq i <n,
\]
a Pareto optimal solution. In order to obtain a population that contains for all values
\[
\sum_{k=1}^j w_k, ~~~~0 \leq j \leq i+1 \leq n,
\]
a Pareto optimal solution, the algorithm can choose the Pareto optimal solution $X$ of weight
\[
\sum_{k=1}^i w_k
\]
 for mutation and insert the variable $x_{i+1}$ at any position of the tree $X$. 
The probability of such a step is at least
\[
\frac{1}{n+1} \cdot \frac{1}{3} \cdot \frac{1}{2n} = \Omega(1/n^2)
\]
and the expected waiting time for such a step is therefore $O(n^2)$. A population containing for each Pareto optimal objective vector one single solution is obtained after at most $n$ such steps which implies that the expected optimization time is upper bounded by $O(n^3)$.
\end{proof}

\section{Conclusions}

With this paper we have contributed to the theoretical understanding of genetic programming. Such algorithms often encounter the bloat problem which means that syntax trees grow during the optimization process without providing additional benefit. One way of dealing with the bloat problem is to take the complexity as an additional criterion to measure the quality of a solution. We have studied the \oneonegp on multi-criteria fitness functions for \WORDER and \WMAJ. These problems are generalizations of \ORDER and \MAJ analyzed in~\cite{GPFOGA11} and we have given better upper bounds than the ones presented in \cite{GPFOGA11}.

Afterwards, we analyzed a  multi-objective genetic programming algorithm called  \smogp. 
This algorithm is inspired by the SEMO algorithm which has been considered in several studies on the computational complexity of evolutionary multi-objective optimization. We are optimistic that it can serve for further studies on the computational complexity of multi-objective genetic programming. 
We have shown that the Pareto fronts of \MOORDER and \MOMAJ are computed by \smogp within a small amount of time.
Furthermore, we have extended our investigations to \WMOORDER and \WMOMAJ which can encounter an exponential number of trade-off objective vectors. However, the size of the Pareto front is linear with respect to the problem dimension and \smogpsingle computes this Pareto front in expected polynomial time when starting with a non-redundant solution. 

We finish with two interesting topics for future work.
\begin{itemize}
\item Determine the expected optimization time of \oneonegpmulti which chooses $k$ according to 1+Pois(1) on \WMOORDER and \WORDER, \WMOMAJ, and \WMAJ. 
\item Determine the expected optimization time of \smogpmulti on \WMOORDER and \WMOMAJ.
\end{itemize}

\section*{Acknowledgement}
The author thanks Christian Igel, Aneta Neumann, Una-May O'Reilly, and Markus Wagner for interesting discussions on the topic of this paper.

\balance
\bibliographystyle{abbrv}
\bibliography{main,unamay}

\end{document}